\documentclass[10pt,twocolumn,letterpaper]{article}

\usepackage{wacv}              % To produce the CAMERA-READY version
\usepackage{graphicx}
\usepackage{amsmath}
\usepackage{amssymb}
\usepackage{booktabs}
\usepackage{amsfonts}       % blackboard math symbols
\usepackage{mathtools}
\usepackage{amsthm}
\usepackage{xcolor}

\usepackage{physics}
\usepackage{wrapfig}
\usepackage{pifont}
\usepackage{caption}
\usepackage{lipsum}
\usepackage{multirow}
\usepackage{algorithm}
\usepackage{color}
\usepackage[table]{xcolor}

\usepackage{amsmath,amsfonts,bm}

\def\eqref#1{equation~\ref{#1}}
\def\1{\bm{1}}

\def\rvh{{\mathbf{h}}}
\def\rvu{{\mathbf{i}}}

\def\rvu{{\mathbf{u}}}
\def\rvv{{\mathbf{v}}}
\def\rvw{{\mathbf{w}}}
\def\rvx{{\mathbf{x}}}

\def\rvz{{\mathbf{z}}}

\def\mP{{\bm{P}}}

\DeclareMathAlphabet{\mathsfit}{\encodingdefault}{\sfdefault}{m}{sl}
\SetMathAlphabet{\mathsfit}{bold}{\encodingdefault}{\sfdefault}{bx}{n}

\DeclareMathOperator*{\argmax}{arg\,max}

\theoremstyle{plain}
\newtheorem{theorem}{Theorem}[section]
\newtheorem{proposition}[theorem]{Proposition}

\theoremstyle{definition}

\theoremstyle{remark}

\definecolor{wacvblue}{rgb}{0.21,0.49,0.74}
\usepackage[pagebackref,breaklinks,colorlinks,allcolors=wacvblue]{hyperref}

\def\wacvPaperID{1191} % *** Enter the WACV Paper ID here
\def\confName{WACV}
\def\confYear{2026}

\title{Scalpel: Fine-Grained Alignment of Attention Activation Manifolds 
via Mixture Gaussian Bridges to Mitigate Multimodal Hallucination}

\author{Ziqiang Shi$^{\star}$\thanks{Corresponding author: shiziqiang@fujitsu.com}, Rujie Liu$^{\star}$, Shanshan Yu$^{\dagger}$, Satoshi Munakata$^{\dagger}$, Koichi Shirahata$^{\dagger}$\\
$^{\star}$ Fujitsu Research \& Development Center Co.,LTD., Beijing, China\\
$^{\dagger}$ Fujitsu Limited, Tokyo, Japan
}

\begin{document}
\maketitle

\begin{abstract}
Rapid progress in large vision-language models
(LVLMs) has achieved unprecedented performance
in vision-language tasks. However, due to the strong prior of large language models (LLMs) and misaligned attention across modalities,
LVLMs often generate outputs inconsistent with
visual content - termed hallucination. To address
this, we propose \textbf{Scalpel}, a method that reduces
hallucination by refining attention activation
distributions toward more credible regions. Scalpel
predicts trusted attention directions for each head
in Transformer layers during inference and adjusts
activations accordingly. It employs a Gaussian mixture
model to capture multi-peak distributions of
attention in trust and hallucination manifolds, and
uses  entropic optimal transport (equivalent to 
Schr{\"o}dinger bridge problem) to map Gaussian components
precisely.
During mitigation, Scalpel dynamically
adjusts intervention strength and direction based
on component membership and mapping relationships
between hallucination and trust activations.
Extensive experiments across multiple datasets and
benchmarks demonstrate that Scalpel effectively
mitigates hallucinations, outperforming previous
methods and achieving state-of-the-art performance.
Moreover, Scalpel is model- and data-agnostic,
requiring no additional computation, only a single
decoding step.
\end{abstract}

%%%%%%%%% BODY TEXT
\section{Introduction}
\label{sec:intro}

Large visual language models (LVLMs) have become essential tools for handling diverse  
visual tasks and performing complex visual question-answering due to their strong  
capabilities in content understanding and 
generation~\cite{liu2023visual,wang2024qwen2,zhu2023minigpt,yang2025mmada}. 
Despite these  
advancements, LVLMs often suffer from the ``hallucination'' problem, where generated  
text appears plausible but misaligns with image content or even fabricates elements  
not present in the input~\cite{liu2024mia,leng2024mitigating,sun2023aligning,leng2024mitigating,huang2024opera,yin2025clearsight,tu2025attention}. This issue significantly undermines their reliability in  
critical domains such as healthcare, autonomous driving, and security monitoring.  

\begin{figure*}[th]
\begin{center}
   \includegraphics[width=1.0\linewidth]{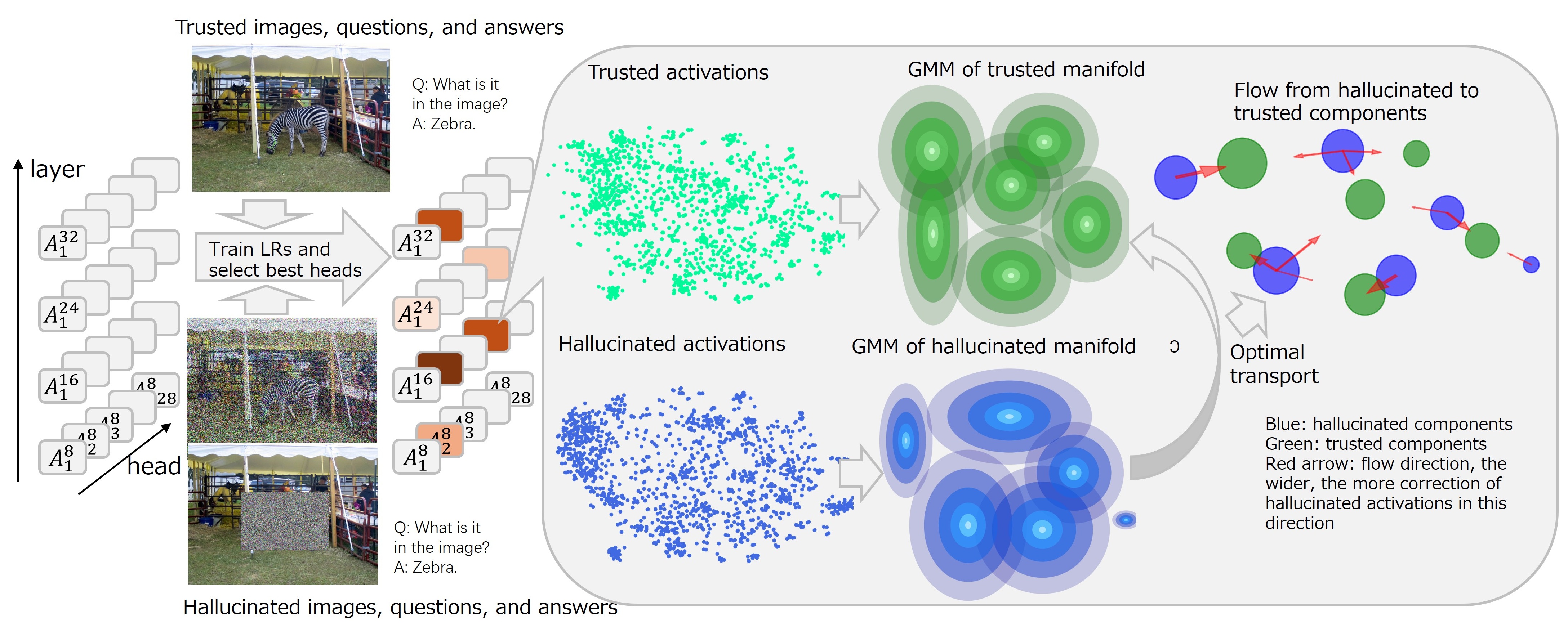}
\end{center}
   \caption{
Schematic diagram illustrating the principle of the Scalpel method. First, the trusted 
and hallucinated attention activations of all heads are obtained by inputting 
both correct and hallucinated data into the LVLM. These activations serve 
two purposes. The first is to identify the heads with the highest hallucination 
discrimination capability using a classifier. The second is to analyze the 
distribution of credible and hallucinated attention activations across tokens, 
thereby determining the influence of each individual token on each head.
   }
\label{fig:framework}
\end{figure*}

The hallucinations in current LVLMs are generally attributed to two main factors:  
the strong language priors inherited from pre-trained large language 
models (LLMs)~\cite{liu2024mia,leng2024mitigating},  
and the model's tendency to over-attend to irrelevant visual tokens or system tokens  
unrelated to the given instruction~\cite{sun2023aligning,wang2024eaco,yin2025clearsight,tu2025attention}. To mitigate this, researchers have explored  
reinforcement learning (RL) strategies that fine-tune LVLMs using 
high-quality~\cite{liu2024mia,lu2025dama,sun2023aligning,chen2025perturbollava},  
hallucination-free feedback—either human-labeled or AI-generated—to improve alignment  
between vision and language while reducing reliance on LLM priors. However, these  
methods demand substantial computational resources for annotation and training,  
making them impractical in resource-constrained environments. Consequently, recent  
studies have shifted focus toward inference-time optimization techniques that do  
not require retraining, such as contrastive 
decoding~\cite{leng2024mitigating,huang2024opera,favero2024multi,woo2024don}, 
which reduces hallucinations by  
adjusting logit scores. 
Although training-free, these methods still fall short in  
effectiveness and often introduce noticeable latency due to multi-step decoding.  

While both approaches yield some success, they rarely investigate the internal  
mechanisms of the LVLM itself to uncover the root causes of hallucinations. In contrast,  
this paper delves into the role of attention activation across different heads and  
layers within the LVLM's Transformer~\cite{vaswani2017attention} architecture during the generation of hallucinated  
content. We discover that individual attention heads contribute differently to  
hallucinations. By precisely identifying those responsible and applying targeted  
interventions, we can significantly enhance output quality and multimodal reasoning  
performance.  

Our key contributions are summarized as follows:  
\begin{itemize}
\item  We propose Scalpel, a novel, training-free, plug-and-play method that enables  
customized correction of attention activations at the token level without compromising  
the LLM's knowledge capacity, effectively reducing hallucinations in LVLMs.  
\item To achieve precise attention correction per token, Scalpel employs Gaussian 
mixture  
models (GMMs) to represent trusted and hallucinated attention activation manifolds, then  
applies Schr{\"o}dinger bridge problem solving 
theory to compute optimal transport between these GMMs as shown in 
Figure~\ref{fig:framework}.  
This allows for component-specific correction based on the current token's activation.  
\item Experimental results on POPE~\cite{li2023evaluating} and MME~\cite{fu2023mme}
benchmarks demonstrate that Scalpel significantly  
improves performance on LLaVA-1.5~\cite{liu2023visual} and 
Qwen2.5-VL~\cite{bai2025qwen2}, proving its 
model-agnostic and task-agnostic  
applicability across multiple databases.
\end{itemize}

\section{Related Work}
\label{sec:related_work}

\subsection{Hallucination in LVLMs}

Although LVLMs, built upon open-source language models such as 
LLaMA~\cite{touvron2023llama} 
and Vicuna~\cite{chiang2023vicuna}, have achieved effective multimodal fusion of 
text and images—greatly enhancing cross-modal 
understanding and generation—they still face a critical challenge: hallucination. 
Hallucination refers 
to factual inconsistencies between the generated text and the input image, 
often manifesting as fabricated 
objects or inaccurate descriptions of attributes and relationships. This issue 
primarily stems from two 
key factors: (1) an over-reliance on prior knowledge (e.g., biases in 
training data or inherent language 
priors within LLMs)~\cite{liu2024mia,leng2024mitigating}, and (2) limitations in 
visual encoder localization, misalignment of multimodal 
information~\cite{sun2023aligning,wang2024eaco}, or suboptimal attention modeling 
during decoding~\cite{yin2025clearsight,tu2025attention}, all of which hinder 
accurate associations 
between inputs and outputs.

\subsection{Hallucination Mitigation in LVLMs}

LVLM hallucination mitigation falls into two categories:  
training-based and training-free post-processing methods.  
Training-based approaches include:  
MIA-DPO~\cite{liu2024mia} creates multi-image data via grid collages from single inputs;  
LLaVA-RLHF~\cite{sun2023aligning} reduces reward hacking through fact-grounded reward models;  
DAMA~\cite{lu2025dama} dynamically optimizes training by aligning data difficulty with response strategies;  
PerturboLLAVA~\cite{chen2025perturbollava} weakens language prior reliance via adversarial text perturbations;  
EACO~\cite{wang2024eaco} achieves alignment with only 5,000 images (high efficiency);  
AMP~\cite{zhang2024automated} enhances fine-grained recognition through multi-level preference optimization.  

Training-free post-processing improves generation accuracy without retraining.
These methods fall into two categories:
contrastive decoding~\cite{leng2024mitigating,huang2024opera,favero2024multi,woo2024don}
and attention reallocation~\cite{yin2025clearsight,woo2024don,tu2025attention,zhang2024seeing}.
Contrastive decoding reduces bias via cross-image comparison (VCD~\cite{leng2024mitigating}),
uses penalties/rollback for refinement (OPERA~\cite{huang2024opera}),
and strengthens vision-language alignment (M3ID~\cite{favero2024multi}).
Attention reallocation improves responsiveness with blind-token calibration (AVISC~\cite{woo2024don}),
provides plug-and-play enhancement (ClearSight~\cite{yin2025clearsight}),
optimizes resource use through recycling (AttnReal~\cite{tu2025attention}),
and broadcasts attention matrices for focus enhancement (EAH~\cite{zhang2024seeing}),
achieving minimal computational overhead.

Our approach fundamentally differs by leveraging inference-time
intervention~\cite{li2023inference,chen2024ict}. The core procedure
comprises: (1) identifying factual-response directions in activation
space; (2) steering activation vectors toward these directions during
reasoning. Compared to existing methods, this technique offers:
negligible computational overhead; strong interpretability through
explicit activation manipulation; and significant hallucination
reduction while preserving generation quality.

\section{Methodology}
\label{sec:method}

\subsection{Background and Notation}
\label{sec:notation}
  
The input to an LVLM is multimodal, consisting of both text and visual components.  
For simplicity, we do not distinguish between system prompts and user instructions 
in text input.  
Let $\rvx^t = \{x_1, x_2, \dots, x_M\} $ denote the token sequence encoded 
by the text encoder,  
representing a textual feature sequence of length $M$. Similarly, let  
$ \rvx^v = \{x_{M+1}, x_{M+2}, \dots, x_{M+N}\} $ represent the visual 
token sequence encoded  
by the vision encoder, with length $N$. These sequences are concatenated into a unified  
input sequence $\rvx = \{x_1, x_2, \dots, x_{M+N}\} $ before being fed into 
the LVLM.  

The LVLM processes this  input token sequence $\rvx$ through $L$ Transformer layers. The output 
at layer $l+1$,  
denoted as $\rvh^{(l+1)}$, is computed using multi-head 
self-attention~\cite{vaswani2017attention}:  
\begin{equation}
\rvh^{(l+1)} = \rvh^{(l)} + \sum_{n=1}^{N_h} 
%\mQ_n^{l} 
\mathcal{A}_n^{l}( \rvh^{(l)})\mP_n^{l}
\end{equation}
where $N_h$ is the total number of heads in a Transformer layer,
$d$ is the dimension of the operation space of each head;
$\mathcal{A}_n^{l}(\cdot)$ is 
the attention operator of the $n$-th head in the 
$l$-th layer;
$\mP_n^{l}\in \mathbb{R}^{d \times dN_h}$ 
maps the activated output of attention  head
back to the operation space 
of the $l$-th Transformer layer.
Each layer performs self-attention to capture interactions between text 
and visual features. At the final layer (layer $L$), the hidden 
states $\rvh^{(L)}$ are passed through an affine layer 
to produce logits of size equal to the 
vocabulary, followed by Softmax normalization to 
generate the probability distribution for the next token $y_t$:  
$$
p(y_t | y_{<t}) = \text{Softmax}(\text{Affine}(\rvh_t^{(L)}))
$$
where $ y_{<t} $ represents the previously generated tokens.

\subsection{Modeling Trusted and Hallucinated Manifolds of Attention Activations}
\label{sec:manifolds}

Our method, Scalpel, is based on two key findings  
from studies on LLMs and LVLMs.  
First, these models often know the correct answer but  
fail to express it clearly~\cite{wang2020language,kadavath2022language,saunders2022self}.  
Second, their activation space has key directions that  
guide truthful responses~\cite{burns2022discovering,subramani2022extracting}.  
Earlier methods like ITI~\cite{li2023inference} and ICT~\cite{chen2024ict} reduce  
hallucinations by steering attention heads toward trusted  
activation spaces. Yet, they use fixed correction  
directions for all activations in a head, which may not be ideal.  
Scalpel improves this by customizing intervention  
directions for different activations during inference,  
even within the same attention head.  
To do this efficiently, we separately discretize and  
tokenize trusted and hallucinated manifolds, then  
create a one-to-one mapping between them.  
GMMs~\cite{reynolds1995robust} are well-suited for this task.  
We use GMMs to model attention activations from both  
trusted and hallucinated data, approximating their manifolds.

Trusted activation manifolds are built from valid 
(image, \textbf{Q}, \textbf{A}) triplets (e.g., Q: What is 
in the image? A: Zebra.), extracting attention 
activations across all layers/heads, as shown in the left part of Figure~\ref{fig:framework}. Hallucinated 
manifolds are created by perturbing images while 
keeping QA pairs correct, or by altering specific 
image regions (e.g., zebra bounding boxes). 
For each attention head/layer, we extract three 
manifolds: (1) trusted, (2) image-perturbed 
hallucinated, (3) object-perturbed hallucinated. 
Activation distributions differ slightly between
image and object levels, leading to different
intervention directions. At the object level,
as shown in Figure~\ref{fig:intervention}, most background remains
unchanged, so only specific components need
large-scale intervention. However, intervention
principles are identical, so we omit level
distinctions in the following.
Our goal is to map hallucinated to trusted manifolds, 
enabling corrective adjustments to activations.
Let \(\rho_H\) denote the hallucinated manifold's 
distribution with sample \(\rvz_0\), and \(\rho_T\) 
the trusted with \(\rvz_1\). The optimal mapping is 
given by the entropic optimal transport (EOT) problem~\cite{peyre2019computational}:
\begin{equation}  
\label{eq:eot}  
\min_{\pi \in \Pi(\rho_H, \rho_T)} \int \|\rvz_0 - \rvz_1\|^2 d\pi 
- \epsilon h(\pi)  
\end{equation}  
where \(\pi(\rvz_0, \rvz_1)\) is the transport plan 
(coupling), \(\Pi(\rho_H, \rho_T)\) the joint distributions 
with these marginals, and \(h\) the differential entropy:
\begin{equation}  
\label{eq:diff_entropy}  
h(\rho) \triangleq -\int \rho(\rvz) \log \rho(\rvz) d\rvz.  
\end{equation}

Directly solving the EOT problem (Eq.~\ref{eq:eot}) 
from hallucinatory to trusted activations 
is difficult. Leonard reformulates it as 
an equivalent Schr{\"o}dinger bridge problem 
(SBP) for tractability~\cite{leonard2014survey}: 
\begin{align} 
\label{eq:sbp}  
&\min_{\rvu \in \mathcal{U}} \mathbb{E}_{t \sim \rho_t} 
\left[ \int_0^1 \frac{1}{2\epsilon} \| \rvu_t(\rvz_t) \|^2 \, dt \right], \\  
&d\rvz_t = \rvu_t(\rvz_t) \, dt + \sqrt{\epsilon} \, d\rvw_t, \\  
&\rvz_0 \sim \rho_H, \quad \rvz_1 \sim \rho_T,  
\end{align}  
where \(\mathcal{U}\) is the set of adapted 
finite-energy controls (i.e., drift in 
diffusion/SDEs). The goal is minimal-energy 
control ensuring initial distribution \(\rho_H\) 
and terminal \(\rho_T\). We solve this SBP to 
derive the optimal path from hallucination 
to trusted manifolds.

\subsection{Optimal Transport Mapping Between Trusted and Hallucinated GMMs}
\label{sec:transport}

Both trusted/hallucinated GMMs contain multiple components. 
When an activation belongs to a hallucinated component, 
mapping requires not only identifying its trusted 
counterpart but determining the optimal transformation 
path. To address this, we propose an alignment algorithm 
superior to random matching. Intuitively, minimal 
corrections preserve data manifold integrity, 
thus we seek minimum-cost flow mapping between components, as depicted in the right part of Figure~\ref{fig:framework}.

This process is performed independently per attention head. 
Let hallucinated GMM (for a head) be:
\begin{equation}
\label{eq:gmm0}
\rvz_0 \sim \rho_H \approx \sum_{i=1}^{N_0} w^i_0 \mathcal{N}(\mu^i_0, \Sigma^i_0),
\end{equation}
and trusted GMM:
\begin{equation}
\label{eq:gmm1}
\rvz_1 \sim \rho_T \approx \sum_{j=1}^{N_1} w^j_1 \mathcal{N}(\mu^j_1, \Sigma^j_1),
\end{equation}
where \(\mu^i_0, \Sigma^i_0\) and \(\mu^j_1, \Sigma^j_1\) are 
mean/covariance parameters \(\forall i,j\). \(\rvz_0\), \(\rvz_1\) 
represent hallucinated/trusted activations.

Eq.~(\ref{eq:sbp}) becomes optimal mapping between GMMs:
\begin{equation}
\label{eq:sb_gmm_problem}
\min_{\rvu \in \mathcal{U}} \mathbb{E}
\left[ \int_0^1 \|\rvu_t(\rvz)\|^2 dt \right], \quad 
d\rvz = \rvu_t(\rvz) dt + \sqrt{\epsilon} d\rvw
\end{equation}
with \(\rvz_0\), \(\rvz_1\) satisfying Eq.~(\ref{eq:gmm0}) and~(\ref{eq:gmm1}).

\begin{figure*}[htbp]
  \centering
  \begin{subfigure}{1.0\textwidth}
      \centering
      \includegraphics[width=\textwidth]{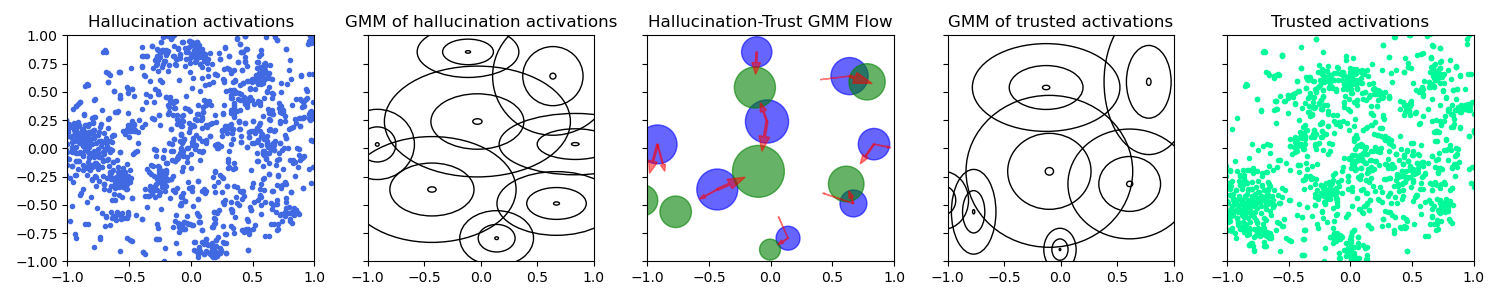}  % 插入第一张图片
      \caption{This comparison shows the distribution of trusted and 
      hallucinated attention activations at the image level. 
      The activations are extracted from the 73rd attention head in layer 1 of LLaVA-1.5. As seen in the middle image, most of the eight hallucination components require 
      noticeable interventions, varying in both direction and scale.}  % 第一张图片的标题
      \label{fig:intervention_image}
  \end{subfigure}
  \hfill
  \begin{subfigure}{1.0\textwidth}
      \centering
      \includegraphics[width=\textwidth]{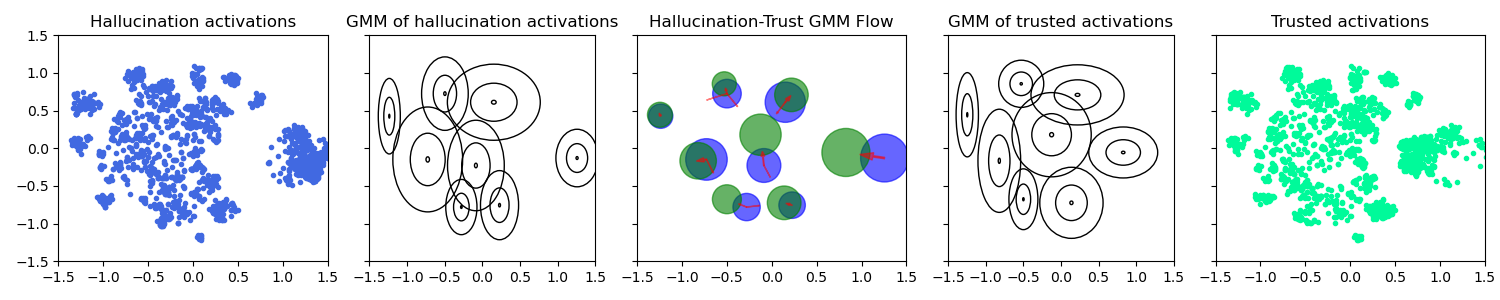}  % 插入第二张图片
      \caption{
This comparison shows distributions of trusted and
hallucinated object-level attention activations,
illustrating GMM transitions. Activations originate
from LLaVA-1.5's 12th layer, 62nd head. Of 8
hallucinated components, 6-7 require minimal intervention
while the rightmost needs substantial correction.
This occurs because hallucinations typically affect
single bounding boxes, leaving other image areas intact -
as shown in Figure~\ref{fig:framework} (zebra example).
}  % 第二张图片的标题
      \label{fig:intervention_object}
  \end{subfigure}
  \caption{Comparison of the distributions of trusted and 
  hallucinated attention activations at the image and object levels, 
  including the component composition of GMMs and the transition flow 
  from the  hallucinated GMMs to the trusted GMMs. For simplicity, we use 
  t-SNE~\cite{van2008visualizing} to map the original attention activations 
  into a two-dimensional space before performing all subsequent operations.}  % 总体图的标题
  \label{fig:intervention}
\end{figure*}

This formulation physically avoids excessive 
perturbation of LVLM attention activations, 
preserving visual-language alignment. Instead, 
it achieves hallucinated-to-trusted state 
shifts via minimal integral intervention.

The Schr\"{o}dinger bridge solves optimal 
transport between distributions. Bunne et al. 
\cite{bunne2023schrodinger} derived a closed-form 
Gaussian bridge solution for single Gaussians 
(i.e., \(N_0=N_1=1\) in Eq.~\ref{eq:sb_gmm_problem} 
constrained by Eq.~\ref{eq:gmm0},\ref{eq:gmm1}). 
This yields the most probable diffusion path 
under Brownian motion, with Gaussian marginals 
at all times. Rapakoulias et al. 
\cite{rapakoulias2024go} extended this to GMMs 
by combining Gaussian bridges via linear 
programming, ensuring theoretical feasibility 
and performance bounds.

Thus, we establish optimal hallucination-to-trust 
GMM mapping (Proposition 1), maximally 
preserving LVLM vision-language alignment.

\begin{proposition} \label{prop:sbp_gmm}  
(Optimal GMM bridge policy for hallucination-to-trust transition)  
For the GMM bridge problem in Eq.~(\ref{eq:sb_gmm_problem}),  
let $\rvu_{t|ij}$ be the optimal control policy for the  
Gaussian bridge between hallucinated component $i$  
($\mathcal{N}(\mu_0^i,\Sigma_0^i)$) and trusted component $j$  
($\mathcal{N}(\mu_1^j,\Sigma_1^j)$), with induced flow $\rho_{t|ij}$  
and optimal cost $J_{ij}$.  

Consider this linear program:  
$$
\min_{\lambda} \sum_{i,j} \lambda_{ij} J_{ij}
$$  
under optimal transport constraints with $\lambda_{ij}\geq 0, \forall i,j$:  
\begin{align*}  
\sum_{j=1}^{N_1} \lambda_{ij} &= w_0^i, & i &= 1,\dots,N_0, \\  
\sum_{i=1}^{N_0} \lambda_{ij} &= w_1^j, & j &= 1,\dots,N_1. 
\end{align*}

The solution $\lambda_{ij}^*$ gives the optimal lower  
bound for Eq.~(\ref{eq:sb_gmm_problem}). The optimal mixture policy:  
\begin{equation}  
\label{eq:optimal_policy}  
\rvu_t(\rvz) = \sum_{i,j} \rvu_{t|ij}(\rvz)  
\frac{\rho_{t|ij}(\rvz) \lambda_{ij}^*}{\sum_{r,\ell}  
\rho_{t|r\ell}(\rvz) \lambda_{r\ell}^*},  
\end{equation}  
is feasible for (\ref{eq:sb_gmm_problem}) with induced flow:  
$$
\rho_t(\rvz) = \sum_{i,j} \rho_{t|ij}(\rvz) \lambda_{ij}^*.
$$  
\end{proposition}  

\emph{Proof sketch}.  
For system $d\rvz_t = \rvu(\rvz_t)dt + \sqrt{\epsilon}d\rvw_t$,  
the Fokker–Planck–Kolmogorov (FPK)~\cite{kolmogorov1931analytischen}  
equation uniquely determines $\rho_t$'s evolution and enables  
control policy derivation. The parabolic PDE guarantees  
existence/uniqueness of $\rho_t$, forming the basis for  
state-feedback control minimizing quadratic costs. Additive  
Gaussian noise and linear dynamics ensure tractability via  
Riccati equations, yielding closed-form optimal policies. $\qed$  

Per Proposition~\ref{prop:sbp_gmm}, the solution uses 
a mixture strategy weighting conditional policies by 
\(\lambda_{ij}^{*} \rho_{t|ij}(\rvz)\), normalized by 
\(\sum_{i,j} \rho_{t|ij}(\rvz)\lambda_{ij}^{*}\). 
Since \(\rho_{t|ij}(\rvz)\) is Gaussian centered at 
the \((i,j)\)-bridge mean at \(t\), this prioritizes 
strategies with means closer to \(\rvz\).
Applied independently per attention head, the 
Schr\"{o}dinger bridge framework maps hallucinated 
to trusted Gaussian components with minimal cost. 
Here \(\lambda_{ij}^{*}\) represents transport weight 
from hallucinated \(i\) to trusted \(j\). For each \(i\), 
select \(j^* = \arg\max_j \lambda_{ij}^{*}\) as the 
trusted counterpart for mitigation (used next section).
Figure~\ref{fig:intervention} shows optimal transfer examples from  
hallucination to trusted manifold. These cases,  
from different attention heads, demonstrate both  
image-level and object-level interventions.  
We observe that tailored interventions can be  
applied for distinct tokens, enabling greater  
customization.

\subsection{Hallucination Mitigation Based on Optimal Mapping Between Manifolds}
\label{sec:mitigation}

To identify critical attention heads, we train logistic  
regression probes over all $L$ layers and $N_h$ heads of  
the vision-language transformer. Given activation tensors  
$\mathcal{A} \in \mathbb{R}^{B \times L \times N_h \times d}$  
(batch size $B$, dimension $d$), we train a logistic model  
on each head's activations to distinguish hallucinated  
(1) vs factual (0) inputs using cross-entropy loss with  
$L_2$-regularized weights. The accuracy matrix $\mathcal{M}  
\in \mathbb{R}^{L \times N_h}$ reveals top-$k$ heads for  
intervention analysis (Figure~\ref{fig:framework}), where  
dark orange highlights indicate strong hallucination  
discrimination capacity.

The mitigation strategy dynamically adjusts key  
attention heads during answer generation to  
influence outputs. At each step, we extract the  
current token's high-dimensional attention vector,  
reflecting the model's internal state critical for  
next-word prediction in autoregressive models.  
Using pre-trained hallucinated and trusted GMMs  
(with multiple components), and their optimal  
mapping (Section~\ref{sec:transport}), we analyze  
activation distributions to determine most probable  
component membership. This clustering interprets  
activation roles by comparing positions against  
learned patterns.  
Let $\rvz_{\text{current}} = \mathbf{z}_t^{(l,h)} \in \mathbb{R}^d$  
denote the output activation vector at step $t$  
from layer $l$, head $h$. Compute posterior  
probabilities for hallucation GMM components:  
\begin{equation}  
\mathcal{P}(r|\mathbf{z}) = \frac{w^r_0 \mathcal{N}(\rvz|  
\mu^r_0,\Sigma^r_0)}{\sum_{i=1}^{N_0} w^i_0 \mathcal{N}(\rvz|\mu^i_0,\Sigma^i_0)}.  
\end{equation}

Let 
\begin{equation}
\label{eq:belong}
r^* = \underset{r}{\arg\max}\ \mathcal{P}(r \mid \rvz_{\text{current}}) 
\end{equation}
denote the hallucinated Gaussian component to which the current activation 
$\rvz_{\text{current}}$ is most likely assigned, and let 
\begin{equation}
\label{eq:confidence}
c = \max_r \mathcal{P}(r \mid \rvz_{\text{current}})
\end{equation}
represent the corresponding assignment probability.

The intervention uses precomputed Schr\"{o}dinger  
bridge mappings (Section~\ref{sec:transport}). When  
identifying a hallucinated component via Eq.~(\ref{eq:belong}),  
we retrieve its trusted counterpart through $\lambda_{ij}^*$  
in Eq.~(\ref{eq:optimal_policy}) and compute the transfer  
vector as $\rvv_{r^*} = \boldsymbol{\mu}_1^{j^*}-\mu_0^{r^*}$,  
where $j^* = \argmax_j \lambda_{r^*j}$.  

Intervention strength combines base coefficient  
$\alpha_{\text{base}}$ and Eq.~(\ref{eq:confidence}) confidence:  
\begin{equation}
\label{eq:alpha}
\alpha_{\text{dynamic}} = \alpha_{\text{base}} \cdot c
\end{equation}  
Strong matches (c$\approx$1) get stronger interventions,  
while ambiguous cases receive milder adjustments  
to maintain stability. Apply the scaled vector:  
\begin{equation}
\label{eq:intervention}
\rvh^{(l+1)} = \rvh^{(l)} + \sum_{n=1}^{N_h}  
\left[\mathcal{A}_n^{l}(\rvh^{(l)})+ \mathbf{1}_{\text{top-}k}(l,n)  
\alpha_{\text{dynamic}}\rvv_{r^*} \right]\mP_n^{l}
\end{equation} 
Only top-$k$ heads (Figure~\ref{fig:framework}) get  
modified via indicator function $\mathbf{1}_{\text{top-}k}$,  
ensuring minimal perturbation while guiding activations  
toward trusted distributions. This correction applies  
at each generation step to reduce hallucinations  
while preserving generation coherence.

\begin{table*}[!h]\small %\footnotesize 
   \caption{ 
   Performance comparison of Scalpel and other baseline methods on the 
POPE benchmark across three datasets — MSCOCO, A-OKVQA, and GQA — using the 
\underline{LLaVA-1.5-7B} model. \textbf{Bold}  values indicate the best performance on each 
dataset and metric.
   }
 \label{tab:scalpel_llava}
  \centering
  \setlength{\tabcolsep}{4pt}
	\begin{tabular}{|l|ccc|ccc|ccc|}
      \hline
	\multirow{2}{*}{Acc.$\uparrow$/F1$\uparrow$} & \multicolumn{3}{c|}{MS COCO} 
	&  \multicolumn{3}{c|}{A-OKVQA} &  \multicolumn{3}{c|}{GQA} 
	 \\ 
	   \cline{2-10} 
			& Random	& Popular & Adversarial & Random & Popular & Adversarial
				  & Random  & Popular
				  & Adversarial  
				  \\ 
              \hline 
              Vanilla~\cite{liu2023visual}   & 83.29/81.33 & 81.88/80.06 & 78.96/77.57 
              &  83.45/82.56 & 79.90/79.59 & 74.04/75.15 
              & 83.73/82.95 &  78.17/78.37 & 75.08/76.06 \\ 
              VCD~\cite{leng2024mitigating} &  87.73/87.16 & 85.38/85.06 &  80.88/81.3 
              &  86.15/86.34 & 81.85/82.82 & 74.97/77.73 
              & 86.65/86.99 & 80.73/82.24 &  76.09/78.78 \\ 
              OPERA~\cite{huang2024opera}  &  89.20/88.81 & 86.64/86.62 & 81.24/81.38
               & 88.02/84.59 &  83.22/84.67 & 73.82/77.91
               &  88.13/88.91 & 79.27/82.11 &  75.00/78.71 \\ 
                  ICT~\cite{chen2024ict} & 89.1/88.48 & 86.76/86.40 & 83.83/83.84
                  &89.3/89.40  & 83.4/84.45 & 75.56/78.68
                  & 89.3/89.49 & 80.86/82.64  & 77.4/80.11\\ 
              Scalpel (ours) & \textbf{90.67}/\textbf{90.74} & 
              \textbf{87.87}/\textbf{88.36} & \textbf{85.97}/\textbf{86.00}
              & \textbf{89.87}/\textbf{89.93} & \textbf{85.00}/\textbf{85.18} & 
              \textbf{78.40}/\textbf{79.71}
              &  \textbf{89.93}/\textbf{89.87}  & \textbf{84.57}/\textbf{85.31}  
              &  \textbf{81.00}/\textbf{82.09} \\ 
  \hline
  \end{tabular}
\end{table*}

\section{Experiments}
\label{sec:experiments}

\subsection{Benchmarks and Experimental Setup}

\textbf{Polling-based Object Probing Evaluation (POPE).}
POPE~\cite{li2023evaluating} evaluates object 
hallucinations in LVLMs via binary queries 
(e.g., ``Is there a chair?''). Unlike caption-based 
methods, it directly probes object recognition 
and hallucination. The balanced dataset (27K pairs) 
contains 50\% real/50\% absent objects from 
COCO~\cite{lin2014microsoft}, A-OKVQA~\cite{schwenk2022okvqa}, 
and GQA~\cite{hudson2019gqa}. Three sampling strategies: 
random, popular (frequent objects), adversarial 
(challenging cases). Evaluation uses Accuracy and F1.

\begin{table*}[!h]\small %\footnotesize 
   \caption{ 
Performance comparison of Scalpel and the previous 
SOTA method ICT using the \underline{LLaVA-1.5-7B} model on the 
POPE benchmark across three datasets: MS COCO, A-OKVQA, and GQA. Results 
are shown for three configurations: image-level intervention (w/o obj), 
object-level intervention (w/o img), and their combination. 
The outputs marked with a blue background represent the results of our Scalpel.% approach. 
 }
 \label{tab:scalpel_ict_llava1_5_on_pope}
  \centering
  \setlength{\tabcolsep}{4pt}
	\begin{tabular}{|l|ccc|ccc|ccc|}
      \hline
	\multirow{2}{*}{Acc.$\uparrow$/F1$\uparrow$} & \multicolumn{3}{c|}{COCO} 
	&  \multicolumn{3}{c|}{A-OKVQA} &  \multicolumn{3}{c|}{GQA} 
	 \\ 
	   \cline{2-10} 
			& Random	& Popular & Adversarial & Random & Popular & Adversarial
				  & Random  & Popular
				  & Adversarial  
				  \\ 
            \hline 
                  Vanilla~\cite{liu2023visual}  & 83.29/81.33 & 81.88/80.06 & 78.96/77.57 
              &  83.45/82.56 & 79.90/79.59 & 74.04/75.15 
              & 83.73/82.95 &  78.17/78.37 & 75.08/76.06 \\
              \hline 
                  ICT w/o obj& 87.53/86.20 & 86.56/85.31 & 85.1/83.97
                  & 85.33/85.16	&  88.7/88.16	 & 78.7/79.81
                  & 89.06/88.53	 &  84.73/84.69	 & 81.43/81.97 \\ 
           \rowcolor{blue!20}   Scalpel w/o obj & 89.07/88.18 &  87.07/86.24 & 85.77/85.29
              & 89.53/89.36		& 85.50/85.29 &  79.30/79.90
              & 89.67/89.40	 & 84.63/84.70	 & 81.47/82.08 \\ 
              \hline
   ICT w/o img & 89.4/88.85 &  86.83/86.49 &  83.7/83.82 
    & 83.23/84.34 &  89.13/89.26  & 75.6/78.72
                  & 89.2/89.40  & 84.06/83.44 & 74.3/74.91 \\ 
          \rowcolor{blue!20}    Scalpel w/o img & 89.90/90.08 & 87.6/86.97& 85.53/84.66 
              & 89.47/89.27& 85.27/85.22 & 78.47/80.18
              & 89.87/89.81 & 84.47/85.04 & 80.93/81.89\\ 
  \hline
   ICT~\cite{chen2024ict} & 89.1/88.48 & 86.76/86.40 & 83.83/83.84
                  &89.3/89.40  & 83.4/84.45 & 75.56/78.68
                  & 89.3/89.49 & 80.86/82.64  & 77.4/80.11\\ 
    \rowcolor{blue!20}          Scalpel (ours) & 90.67/90.74 & 87.87/88.36 & 85.97/86.00
              & 89.87/89.93 & 85.00/85.18 & 78.40/79.71
              &  89.93/89.87  & 84.57/85.31  &  81.00/82.09 \\ 
  \hline
  \end{tabular}
\end{table*}

\textbf{Multimodal Model Evaluation (MME).} MME~\cite{fu2023mme} comprehensively assesses LVLMs 
across 14 subtasks: 10 perception, 4 cognition. 
Perception includes object existence/count 
(object hallucinations), position/color 
(attribute hallucinations). Cognition covers 
commonsense reasoning (CSR), numerical, translation, 
and code reasoning. Accuracy-based metrics used.

\begin{figure}[th]%[htp]
  %\vspace{-0.4in}
  \centering
  %\begin{center}
  \hspace{-5mm}
  \includegraphics[width=1.0\linewidth]{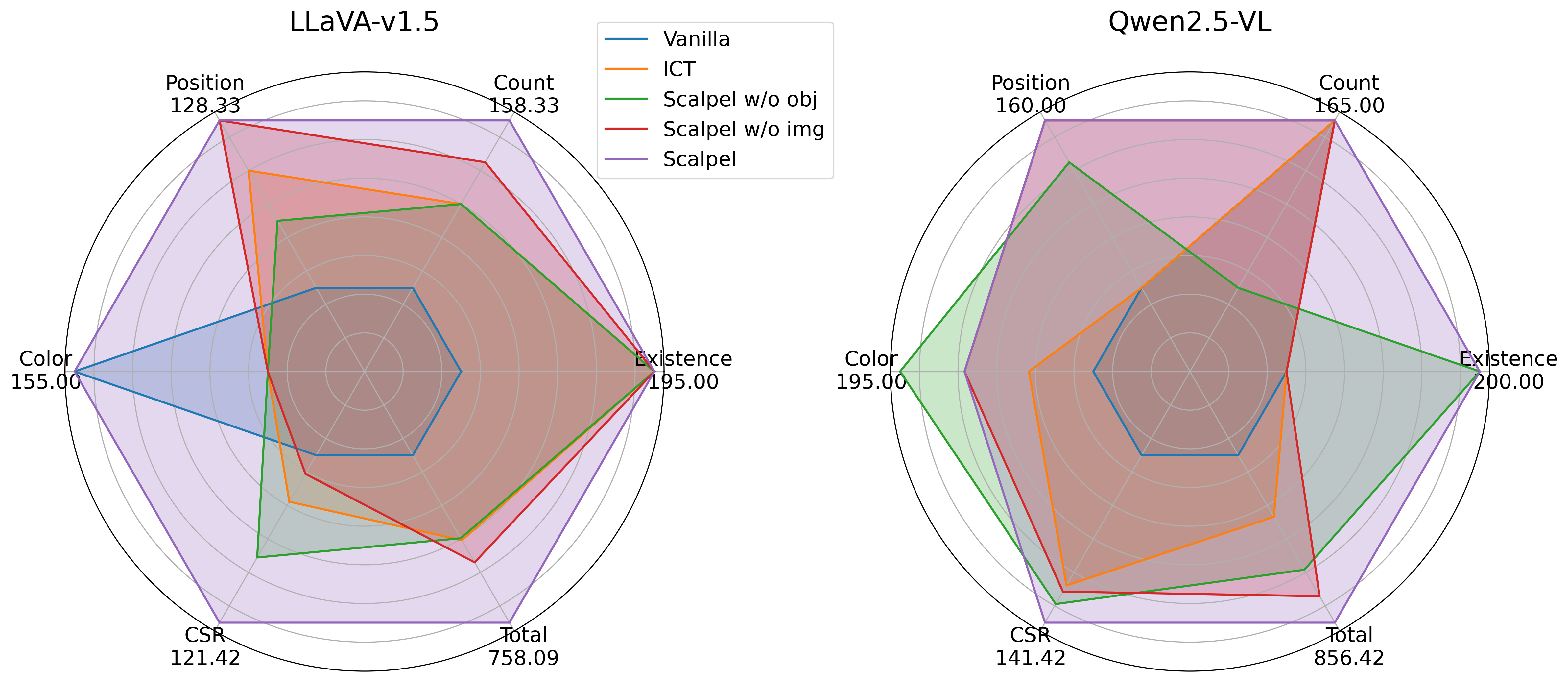}
  \hspace{-5mm}
  %\end{center}
  %\vspace{-2.5mm}
  \caption{
On the MME benchmark, Scalpel outperforms prior
SOTA methods - ICT, Vanilla LLaVA-1.5 and Qwen2.5-VL. The radar
chart highlights improvements across key categories:
existence, location, counting, color perception,
common sense reasoning, and overall performance.
  }
  \label{fig:mme_results}
  %\vspace{-2.5mm}
  \end{figure}
  %\vspace{-0.5cm}

  \textbf{Implementation details.}  
We evaluated Scalpel on two established LVLMs:  
LLaVA-1.5-7B~\cite{liu2023visual} and Qwen2.5-VL-7B~\cite{bai2025qwen2},  
comparing with VCD~\cite{liu2023visual}, OPERA, and ICT\footnotemark[1] hallucination  
mitigation methods. Hyperparameters included: top-$k$  
head selection (Eq.~(\ref{eq:intervention})), intervention  
strength $\alpha_{\text{base}}$ (Eq.~(\ref{eq:alpha})), and GMM component  
count $N_0$ (Eq.~(\ref{eq:gmm0})). To simplify tuning while preserving  
discriminative power, we enforced equal component counts  
for trusted/hallucinated GMMs.  
For head selection, we adopted ICT's framework~\cite{chen2024ict}  
with 1,500 trusted vs. hallucination samples (image/object-level),  
training logistic regression classifiers on attention head  
activations to identify top-$k$ critical heads. This prioritized  
heads showing significant activation patterns during hallucinatory phases.  

\footnotetext[1]{All ICT results are from our re-implementation  
(based on \url{https://github.com/THU-BPM/ICT}) for fair comparison.}  

\subsection{Results and Discussion}

\begin{table*}[!h]\small %\footnotesize 
   \caption{ 
   Performance comparison of Scalpel and the previous 
SOTA method ICT using the \underline{Qwen2.5-VL-7B} model on the 
POPE benchmark across three datasets: MS COCO, A-OKVQA, and GQA. Results 
are shown for three configurations: image-level intervention (w/o obj), 
object-level intervention (w/o img), and their combination. 
The results highlighted with a blue background were produced by our Scalpel.
 }
 \label{tab:scalpel_ict_qwen2_5_vl_on_pope}
  \centering
  \setlength{\tabcolsep}{4pt}
	\begin{tabular}{|l|ccc|ccc|ccc|}
      \hline
	\multirow{2}{*}{Acc.$\uparrow$/F1$\uparrow$} & \multicolumn{3}{c|}{COCO} 
	&  \multicolumn{3}{c|}{A-OKVQA} &  \multicolumn{3}{c|}{GQA} 
	 \\ 
	   \cline{2-10} 
			& Random	& Popular & Adversarial & Random & Popular & Adversarial
				  & Random  & Popular
				  & Adversarial  
				  \\ 
          \hline 
                  Vanilla~\cite{bai2025qwen2} & 85.4/82.97 & 85.13/82.71 & 84.83/82.42
                  &  87.76/86.41 & 86.43/85.15  & 81.5/80.78
                  &  87.1/85.63 & 84/82.74 & 81.6/80.65 \\ 
              \hline 
                  ICT w/o obj& 86.4/84.31 & 86.3/84.27 & 85.7/83.53
                  &  89.93/89.15 & 87.46/86.76 & 82.5/82.54
                  &  88.06/87.01 & 83.9/83.23 & 81.63/81.17\\ 
     \rowcolor{blue!20}         Scalpel w/o obj & 86.27/84.84 &  88.97/87.87 & 86.57/84.87
              & 91.20/90.54 & 88.83/88.37 & 81.70/80.94
              & 88.03/89.67 & 84.43/84.16 & 85.47/85.96 \\ 
              \hline
   ICT w/o img & 85.4/82.97 & 85.4/83.07 & 85.2/82.90
                  & 88.76/87.64 & 86.63/85.50 & 81.7/81.16
                  & 87.16/85.66 & 83.93/82.67  & 81.33/80.40 \\ 
   \rowcolor{blue!20}           Scalpel w/o img & 86.37/84.26 & 85.80/83.60 & 85.37/82.98
              & 88.73/87.66 & 87.17/85.99 & 82.57/82.08
              & 87.73/86.02 & 85.27/85.89 & 83.00/82.29 \\ 
  \hline
   ICT~\cite{chen2024ict} &  87.53/85.84  & 86.76/84.94 & 86.16/84.32
                  & 88.96/87.90 & 87.43/86.39  & 83.6/83.02
                  &  88.96/87.89 &  86.43/85.47 & 84.1/83.53\\ 
      \rowcolor{blue!20}        Scalpel (ours) &  91.17/90.41 & 90.80/90.33 & 88.83/88.49
              & 91.00/90.55 & 90.50/90.41 &  84.97/85.82
              & 89.20/89.84 & 87.20/86.68 &  85.90/85.40 \\ 
  \hline
  \end{tabular}
\end{table*}

Tables~\ref{tab:scalpel_llava},~\ref{tab:scalpel_ict_llava1_5_on_pope},  
and~\ref{tab:scalpel_ict_qwen2_5_vl_on_pope} compare Scalpel with  
leading methods—VCD, OPERA, and ICT—across nine POPE datasets under  
LLaVA-1.5 and Qwen2.5-VL frameworks. Scalpel excels on MS COCO, A-OKVQA,  
and GQA, outperforming prior methods in Random, Popular, and Adversarial  
subsets. Compared to Vanilla, it improves accuracy by 7.61\% and F1 by  
8.88\%, with notable gains in adversarial cases (e.g., +10.87\% F1 on  
MS COCO-Adversarial), demonstrating strong robustness and generalization.  
Against ICT—the previous best—it achieves average relative improvements  
of 2.43\% (Acc.) and 1.81\% (F1), with no performance drop. Largest gains  
appear in Popular and Adversarial subsets (e.g., +4.59\% Acc. on  
GQA-Popular), highlighting its ability to reduce bias and handle hard data.

Table~\ref{tab:scalpel_ict_llava1_5_on_pope} presents detailed analysis  
of Scalpel's modular intervention vs. ICT under LLaVA. Image-Level  
(w/o obj) gains 1.63\% avg F1 (+4.93\% peak on A-OKVQA-Random),  
with strong adversarial performance (+2.47\% Acc. on COCO-Adversarial).  
Object-Level (w/o img) improves 1.85\% avg F1 (+5.85\% peak) while avoiding  
ICT's 84.06 → 84.47 Acc. drop on GQA-Popular. Combined, Scalpel achieves  
max gains (+2.43\% Acc./1.81\% F1) with super-additive effects: joint modules  
deliver +8.22\% F1 on COCO-Adversarial over ICT's single module. Crucially,  
all 18 comparative indicators (9 subsets $\times$ Acc/F1) show strict dominance,  
validating Scalpel's hierarchical design through both superior individual  
modules and synergistic collaboration.

\begin{figure}[th]%[htp]
  %\vspace{-0.4in}
  \centering
  %\begin{center}
  \hspace{-5mm}
  \includegraphics[width=1.0\linewidth]{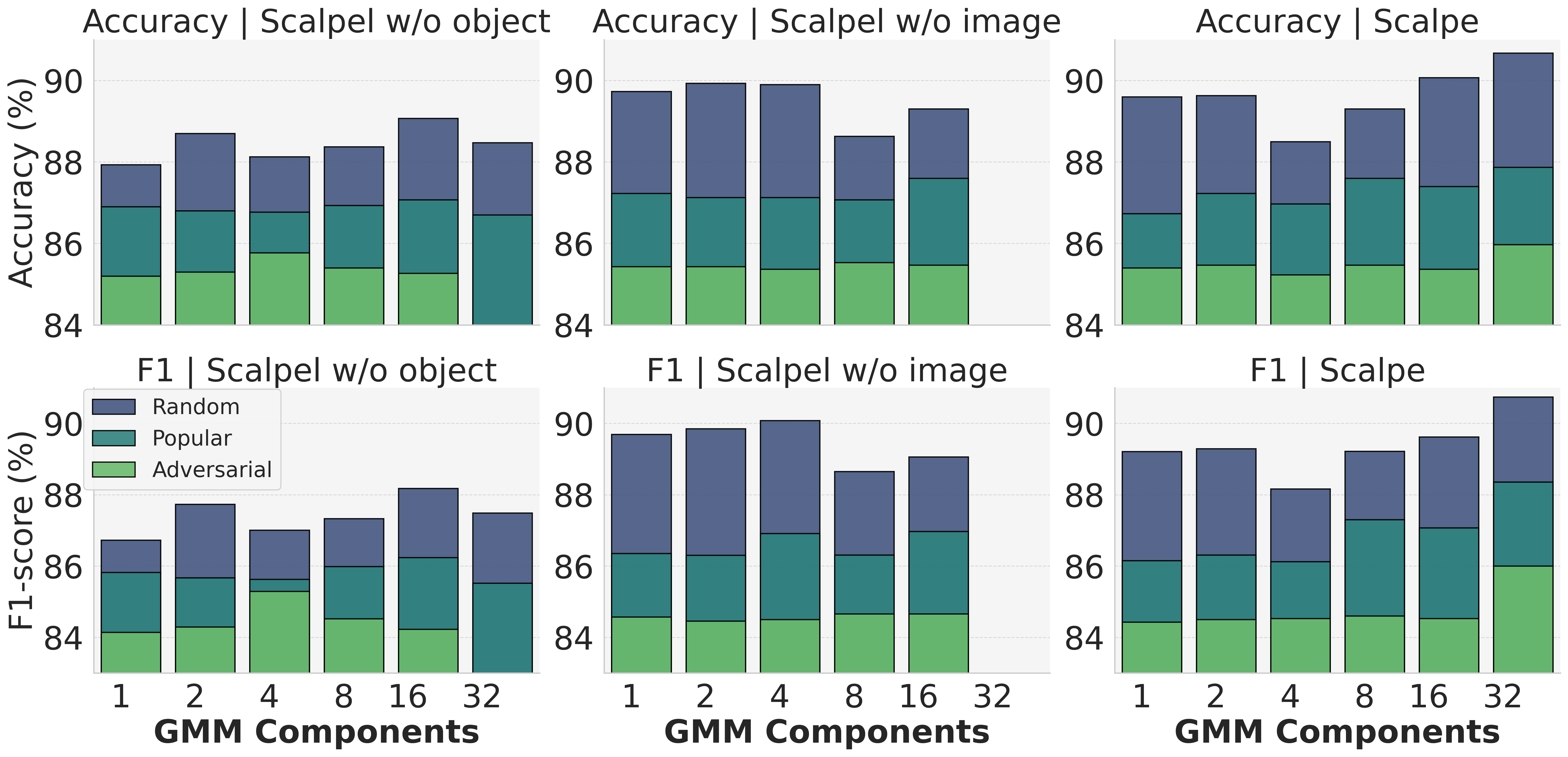}
  \hspace{-5mm}
  %\end{center}
  %\vspace{-2.5mm}
  \caption{
    Ablation study of Scalpel  under different GMM component settings, 
evaluated on the POPE benchmark (MS COCO dataset)
with LLaVA-1.5-7B.
  }
  \label{fig:coco_llava_ablation_components}
  %\vspace{-2.5mm}
  \end{figure}
  %\vspace{-0.5cm}

Table~\ref{tab:scalpel_ict_qwen2_5_vl_on_pope} confirms Scalpel's advantages in  
Qwen2.5-VL with +5.19\% average Acc. (peak +9.94\% on COCO-Random) and  
+7.16\% F1 improvement (peak +9.20\% COCO-Popular), achieving 85.40 F1  
on GQA-Adversarial (vs. 80.65). Modular analysis shows Image-Level (w/o obj)  
+3.07\% avg F1 (A-OKVQA-Random 90.54 vs. 89.15) and Object-Level (w/o img)  
+1.92\% avg F1 (GQA-Popular 85.89 vs. 82.67). Dual-module synergy delivers  
1.67× F1 gain over ICT (Scalpel 2.97\% vs. ICT 1.78\%), maintaining 1.77\%  
absolute advantage on GQA-Adversarial. Adversarial F1 gains exceed ICT by  
1.65× (GQA: 5.89\% vs. 3.57\%).  

Figure~\ref{fig:mme_results} shows multi-dimensional improvements: LLaVA-  
based Scalpel scores 758.09 (+8.52\% vs. Vanilla), excelling in counting  
(+14.46\%), positioning (+14.93\%), and common sense (+11.84\%). Removing  
modules reduces performance (w/o obj: 728.09/+4.23\%; w/o img: 736.66/+5.45\%),  
highlighting image-object synergy. Qwen-based Scalpel scores 851.42 (+6.14\%),  
achieving SOTA in positioning (+14.29\%), color (+5.56\%), and common sense  
(+15.80\%). Module analysis reveals: image removal maintains positional/color  
performance but degrades reasoning (+12.86\% vs. +15.80\%), while object removal  
causes 9.09\% counting drop but improves color (+8.33\%) and reasoning (+14.03\%)  
via reduced interference. Scalpel surpasses ICT (822.14/+2.49\%), with absolute  
dominance in complex reasoning (141.42 vs. 137.85).

\begin{figure*}[th]
\begin{center}
   \includegraphics[width=1.0\linewidth]{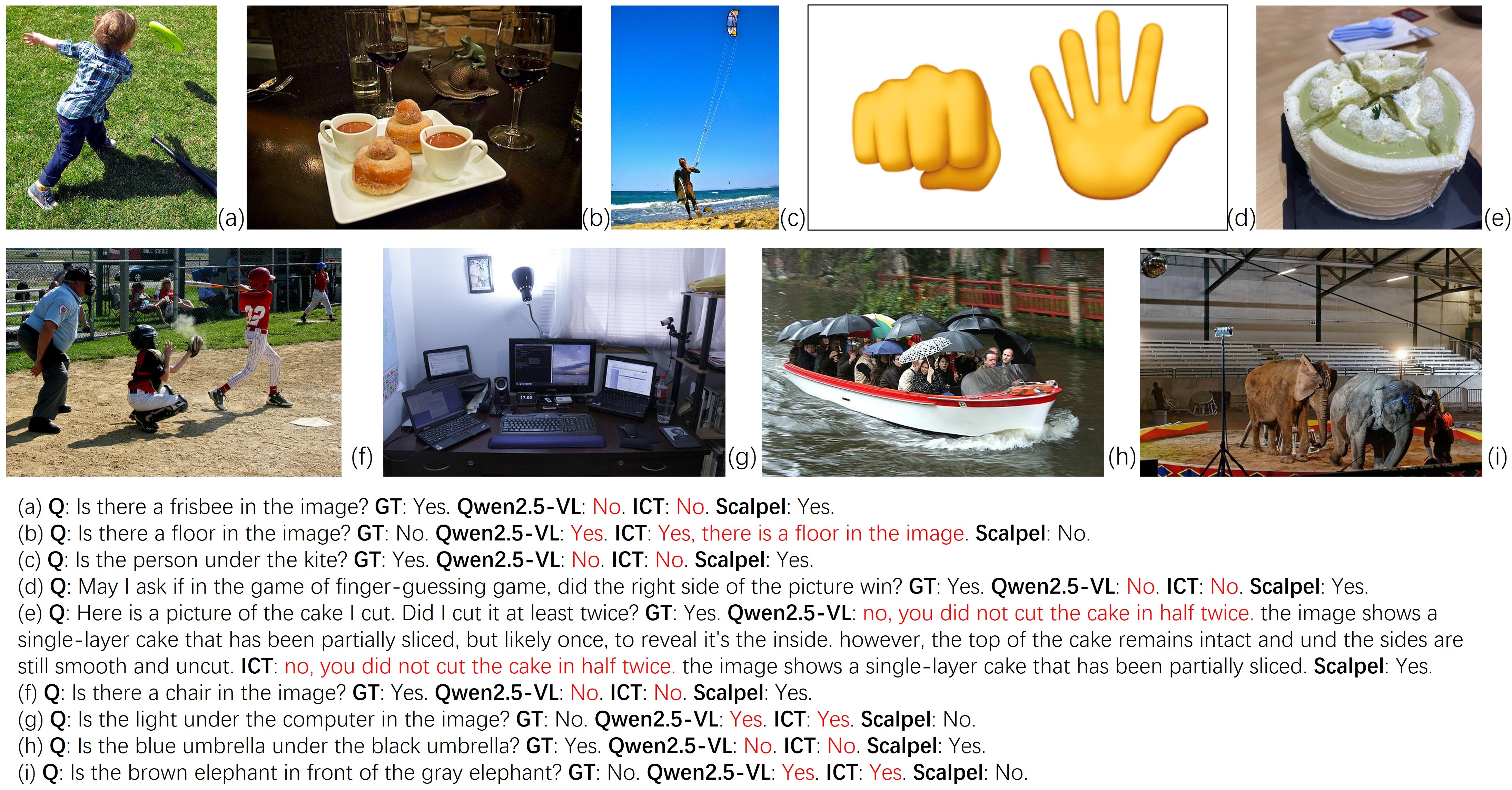}
\end{center}
   \caption{
A comparative analysis of Scalpel, ICT, and the original 
Qwen2.5-VL-7B across selected test cases is presented. 
\textbf{Q} denotes the question, \textbf{GT} represents the ground truth, 
and each method's 
answer follows its name in \textbf{bold}. 
Extended \textcolor{red}{hallucinated responses} are highlighted in  \textcolor{red}{red} for emphasis.
Notably, 
the original questions included the instruction ``Please answer yes or no'' immediately 
after the question mark, which has been omitted to optimize space usage.
   }
\label{fig:case_studu}
\end{figure*}

\subsection{Ablation study}

Scalpel's efficacy was validated through three aspects:  
GMM components (intervention granularity), image-level, and  
object-level interventions. Tables~\ref{tab:scalpel_ict_llava1_5_on_pope},  
~\ref{tab:scalpel_ict_qwen2_5_vl_on_pope}, and  
Figure~\ref{fig:coco_llava_ablation_components} present ablation  
studies. Specifically, Tables~\ref{tab:scalpel_ict_llava1_5_on_pope} and  
~\ref{tab:scalpel_ict_qwen2_5_vl_on_pope} compare Scalpel variants  
(LLaVA-1.5/Qwen2.5-VL) across 9 POPE subsets when removing modules.  
Quantitative results show: adding image-level intervention improves  
LLaVA by +7.46\% Acc./+8.43\% F1 (Qwen: +2.32\%/+3.69\%); object-level  
adds +7.37\% Acc./+8.14\% F1 (LLaVA) and +1.09\%/+1.52\% (Qwen).  
Combined modules show peak gains: +7.61\% Acc./+8.78\% F1 (LLaVA)  
and +4.70\% Acc./+6.48\% F1 (Qwen). These cross-model improvements  
validate dual levels' critical role in hallucination mitigation.

Figure~\ref{fig:coco_llava_ablation_components} analyzes GMM component count's  
impact on hallucination reduction and accuracy. Tests on  
COCO's POPE subsets (Random, Popular, Adversarial) evaluated  
1–32 components. Higher counts improve performance, with  
32 components yielding best results across all subsets.  
Image-level intervention peaks at 16 components for  
Random/Popular, fewer for Adversarial. Object-level  
requires higher counts even in adversarial cases.  
Complementary patterns justify recommending full Scalpel  
with 32 GMM components, balancing granularity needs  
across intervention levels and data complexities.

Figure~\ref{fig:case_studu} presents qualitative comparisons.  
Scalpel surpasses ICT in object perception, spatial analysis, scenario suitability,  
and action trace recognition. First, compared to ICT/Vanilla Qwen2.5-VL, Scalpel enhances image understanding accuracy by 
removing hallucinated elements (e.g., non-visable floors) and detecting missed 
objects (e.g., frisbees/chairs), significantly reducing hallucinations. Second, Scalpel perfectly matches  
GT in object positioning (e.g., occlusion under kites/umbrellas)  
and spatial arrangements (elephant positioning), while ICT  
errors arise from over-reliance on localized features/color cues.   
Third, Scalpel detects subtle visual traces (e.g., cake cuts)  
rather than surface appearances. It also eliminates ICT's  
redundant assumptions/logical inconsistencies (e.g., ``requires  
verificatio''), producing streamlined reasoning paths. These  
advantages reflect superior vision-language alignment and noise  
filtering during complex reasoning.

\section{Conclusion}

To correct hallucinated attention activations in LVLM
Transformers without added computational cost, we
propose Scalpel. This method models hallucinated and
trusted attention activations via Gaussian Mixture
Models (GMMs), yielding hallucinated and trusted GMMs
respectively. The Schr"odinger Bridge (equivalent to
entropic optimal transport) then constructs the minimal
transport-cost correction scheme by treating these
GMMs as marginal distributions. This preserves LVLMs'
data-driven learning capabilities while effectively
suppressing hallucinations. Experiments across datasets
show Scalpel achieves SOTA performance, surpassing
non-customized methods. A promising future direction
is investigating refined correction formulations via
the infinite-component limit of GMMs.

{
    \small
    \bibliographystyle{ieeenat_fullname}
    \bibliography{hallu}
}

\end{document}